\documentclass[11pt]{article}

\usepackage[final]{acl}

\usepackage{times}
\usepackage{latexsym}
\usepackage[T1]{fontenc}
\usepackage[utf8]{inputenc}
\usepackage{microtype}
\usepackage{inconsolata}
\usepackage{graphicx}

\usepackage{hyperref}
\usepackage{url}
\usepackage{booktabs}
\usepackage{amsmath,amssymb,amsfonts,amsthm}
\usepackage{nicefrac}
\usepackage{xcolor}
\usepackage[ruled,vlined]{algorithm2e}
\SetKwInput{KwInput}{Input}
\SetKwInput{KwOutput}{Output}
\usepackage{subcaption}
\usepackage{bbm}
\usepackage{bm}
\usepackage{multirow}
\usepackage[table]{xcolor}
\usepackage{enumitem}

\newtheorem{proposition}{Proposition}

\newcommand{\method}{CARE}
\newcommand{\kmin}{k_{\min}}
\newcommand{\kmax}{k_{\max}}
\newcommand{\pp}{\bm{p}}
\newcommand{\hh}{\bm{h}}
\newcommand{\expo}{\bm{e}}
\newcommand{\Real}{\mathbb{R}}
\newcommand{\Exp}{\mathbb{E}}

\newcommand{\flyavg}{86.0\%}
\newcommand{\careavg}{86.5\%}
\newcommand{\caregain}{0.5}

\newcommand{\qcareavg}{87.9\%}

\newcommand{\qcareoverbest}{0.5}
\newcommand{\extcare}{48.8\%}
\newcommand{\extgain}{0.9}
\newcommand{\savepct}{12\%}
\newcommand{\matchk}{3.5}
\newcommand{\aurocbase}{0.640}
\newcommand{\auroccare}{0.668}
\newcommand{\ecebase}{0.175}
\newcommand{\ececare}{0.186}
\newcommand{\nexperts}{16}
\newcommand{\robustfixed}{50.3\%}
\newcommand{\robustcare}{53.1\%}
\newcommand{\selfixed}{55.1\%}
\newcommand{\selcare}{54.2\%}
\newcommand{\seloodfixed}{52.0\%}
\newcommand{\seloodcare}{53.5\%}
\newcommand{\allocmmlu}{3.6}
\newcommand{\allocmath}{4.6}
\newcommand{\alloccode}{4.4}
\newcommand{\allocsci}{3.4}

\title{Spend Experts Where You Are Unsure:\\ Confidence-Adaptive Routing for Mixture-of-Experts LoRA}

\author{
  \textbf{Tom Saliencro\textsuperscript{1}},
  Rohan Desai\textsuperscript{2},
  Priya Nair\textsuperscript{1},
  Maya Lindqvist\textsuperscript{1},
  Daniel Whitmore\textsuperscript{2}
  \\
  \\
  \textsuperscript{1}University of California, Irvine \\
  \textsuperscript{2}University of Washington \\
  \texttt{saliencro@gmail.com}
}

\begin{document}
\maketitle

\begin{abstract}
Mixture-of-Experts (MoE) variants of Low-Rank Adaptation (LoRA) route every
token to a \emph{fixed} number of experts $k$. Tokens differ in how uncertain
the model is about them, so a single $k$ over-spends on easy tokens and
under-serves hard ones. We observe that the router's output distribution is
already a per-token uncertainty signal: peaked mass indicates confidence, while
a flat distribution indicates ambiguity. We introduce \method{}
(Confidence-Adaptive Routing of Experts), which admits experts in a
\emph{nucleus} fashion. Experts are activated in decreasing router weight until
their cumulative mass reaches a threshold, with a small extension when the
admitted experts disagree. A budget thermostat calibrates the threshold so that
the average number of active experts matches any target. \method{} is a
drop-in, single-forward-pass rule with no extra parameters. Across eight
commonsense benchmarks on LLaMA-3.1-8B and Qwen2.5-7B, as well as math, code,
and knowledge tasks, \method{} improves over fixed top-$k$ MoE-LoRA at matched
compute and matches the fixed-$k{=}4$ baseline while activating \savepct{}
fewer experts. The same confidence and disagreement signals also improve
out-of-distribution detection over MSP, entropy, and multi-pass proxies. We
support the design with nucleus fidelity, budget optimality, and an epistemic
reading of disagreement, and we release code.
\end{abstract}

\section{Introduction}
\label{sec:intro}

Parameter-efficient fine-tuning (PEFT) adapts large language models (LLMs) to
downstream tasks by training a small number of added parameters while keeping the
backbone frozen. Low-Rank Adaptation (LoRA)~\citep{hu2022lora} is the dominant
instance, and a productive line of work injects a \emph{Mixture-of-Experts} (MoE)
structure into LoRA (multiple low-rank ``experts'' selected by a router) to
increase capacity without increasing per-token cost~\citep{dou2024loramoe,li2024mixlora,gao2024higher,zou2025flylora}.
These MoE-LoRA methods share one design choice inherited from sparse
MoE~\citep{shazeer2017outrageously,fedus2022switch}: the router activates a
\emph{fixed} top-$k$ set of experts for every token.

\begin{figure}[t]
\centering
\includegraphics[width=\columnwidth]{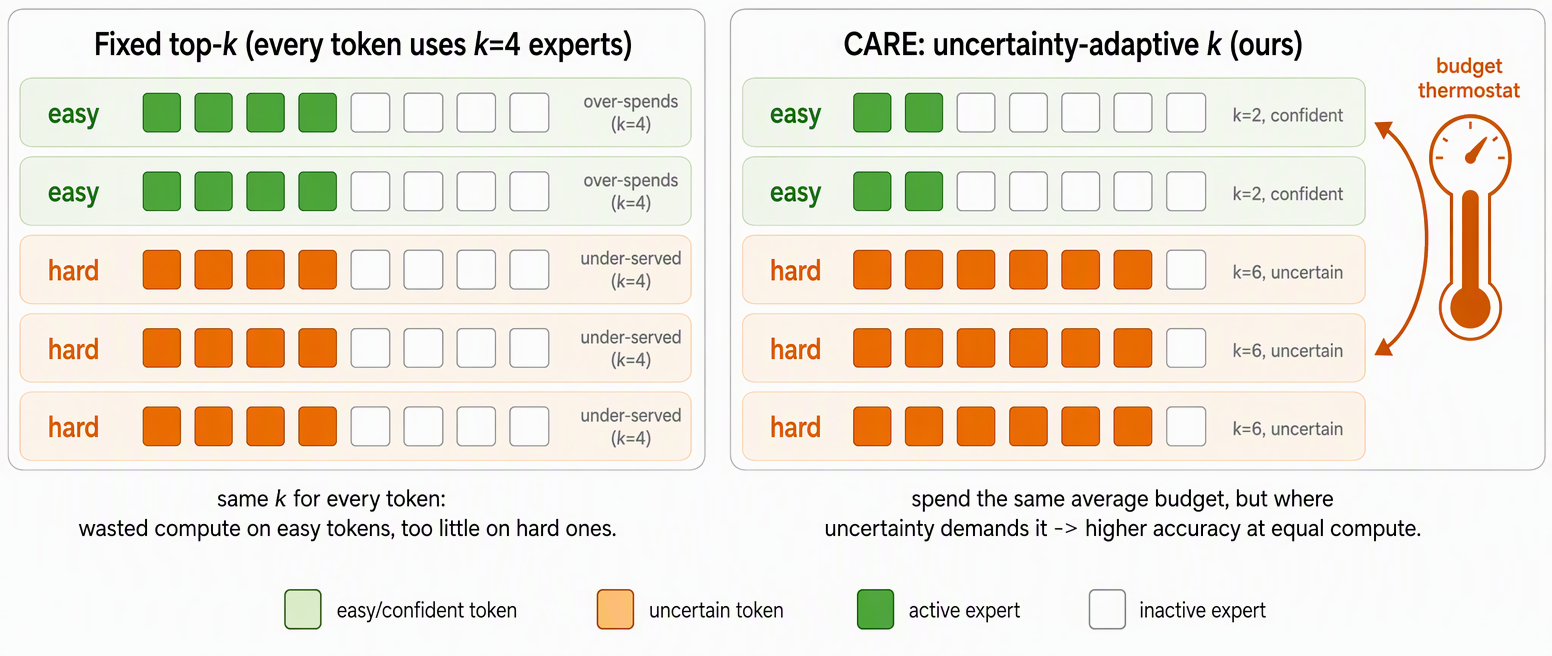}
\caption{\textbf{Fixed top-$k$ spends the same budget on every token; \method{}
spends it where uncertainty demands.} Each row is a token and each square an
expert slot. Left: every token activates $k{=}4$ experts, over-spending on easy
(green) tokens and under-serving hard (orange) ones. Right: \method{} reallocates
the \emph{same} average budget per token according to uncertainty read off the
router.}
\label{fig:teaser}
\end{figure}

A fixed $k$ is convenient but blunt. Tokens are not equally hard: a function word
or an unambiguous continuation needs little adaptation, whereas an ambiguous or
out-of-distribution token may need many experts. Spending the same $k$
everywhere wastes computation on the easy majority while starving the hard
minority that limits accuracy (Figure~\ref{fig:teaser}). Making $k$ adaptive
raises a concrete question: what per-token signal should decide how many experts
to use, and can it be obtained without extra models, extra forward passes, or
posterior sampling?

The router already computes such a signal. After the softmax, it emits a
distribution $\pp$ over experts. When the layer is confident about a token,
$\pp$ is peaked: one or two experts carry almost all the mass. When it is
uncertain, the mass spreads across many experts. Concentration of $\pp$ is
therefore a cheap, per-token proxy for confidence, available in the forward pass
the model already performs. A second signal is available once a few experts are
selected: if the admitted experts disagree in output space, the token is
ambiguous in a way a peaked router can hide, and admitting more experts is
worthwhile.

\method{} (Confidence-Adaptive Routing of Experts) uses these signals to choose
the per-token expert count. Following the spirit of nucleus
sampling~\citep{holtzman2020curious}, it admits experts in decreasing router
weight until their cumulative mass reaches a threshold $\tau$, then admits up to
a few more if the admitted experts disagree. A single global $\tau$ is
calibrated once on a small held-out set so that the \emph{average} number of
active experts equals a target budget $B$; we call this the budget thermostat.
\method{} adds no parameters, requires a single forward pass, and drops into any
MoE-LoRA backbone by replacing its top-$k$ gate.

By moving budget from confident tokens to uncertain ones, \method{} improves
accuracy at a fixed average compute and, equivalently, reaches a target accuracy
with fewer experts. The same signals also serve as uncertainty estimates:
routing concentration provides a confidence score, and expert disagreement
provides an epistemic signal that improves out-of-distribution (OOD) detection.

Prior uncertainty methods for fine-tuned LLMs typically average over many
predictions (deep ensembles, MC-dropout, LoRA ensembles, Bayesian or Laplace
LoRA). That approach is accurate but multiplies inference cost, while methods
that calibrate the router change the training objective. \method{} instead
extracts both a routing decision and an uncertainty estimate from a single
deterministic forward pass of an already-trained MoE-LoRA. It adds no
parameters and leaves the training loop and backbone unchanged: any MoE-LoRA
checkpoint becomes uncertainty-adaptive by replacing its gate, at the cost of a
sort over the $N$ router logits already computed by the model. Our contributions
are:

\begin{itemize}[leftmargin=1.2em,itemsep=1pt,topsep=2pt]
\item We identify the router distribution as a per-token uncertainty signal
for MoE-LoRA and use it to make the number of active experts adaptive
(\S\ref{sec:method}).
\item We propose \method{}, a nucleus expert-admission rule with an epistemic
disagreement extension and a budget thermostat that matches any average compute;
it is single-pass and parameter-free (\S\ref{sec:method}).
\item We give theoretical support: a nucleus fidelity bound, a
confidence-ranking guarantee, budget-allocation optimality, and an ensemble-variance
reading of disagreement (\S\ref{sec:theory}, App.~\ref{app:proofs}).
\item Across two backbones and four task families at matched compute, \method{}
improves accuracy over fixed top-$k$ MoE-LoRA (e.g.\ ${+}\caregain$ on LLaMA
commonsense) and saves \savepct{} of experts at equal accuracy, while improving
OOD detection over MSP, entropy, and multi-pass proxies (\S\ref{sec:exp}).
\end{itemize}

\section{Related Work}
\label{sec:related}

\paragraph{MoE-LoRA and adaptive PEFT.}
LoRA~\citep{hu2022lora} and its adaptive-rank variant
AdaLoRA~\citep{zhang2023adalora} tune a single low-rank branch. MoE-LoRA methods
add multiple LoRA experts with a router: LoRAMoE~\citep{dou2024loramoe} and
MixLoRA~\citep{li2024mixlora} attach experts to feed-forward or attention
modules, \citet{gao2024higher} show deeper layers benefit from more experts, and
FlyLoRA~\citep{zou2025flylora} builds a rank-wise MoE-LoRA. DynMoLE~\citep{li2025dynmole}
stabilizes routing with an entropy-based auxiliary loss. Concurrently,
FRAME~\citep{saliencro2026frame} learns a fractional-Fourier adaptation domain per
expert, but still routes with a fixed top-$k$. All of these commit to a
\emph{fixed} top-$k$ at inference. AdaLoRA adapts the \emph{rank} globally by
importance but keeps computation static across tokens; DynMoLE modulates the
router's sharpness but still activates a fixed number of experts. \method{}
addresses a different axis: it adapts the \emph{number of active experts} per
token by uncertainty, while leaving the experts, the router weights, and the
training objective unchanged. It therefore composes with any MoE-LoRA backbone
as a test-time gate.

\paragraph{Uncertainty in LLMs.}
Deep ensembles~\citep{lakshminarayanan2017simple} and MC-dropout~\citep{gal2016dropout}
estimate uncertainty through multiple stochastic predictions. For LLM fine-tuning,
LoRA ensembles~\citep{wang2023loraensembles,muhlematter2026loraensemble} and
Laplace/Bayesian LoRA~\citep{yang2024bayesian} improve calibration but require
several adapters or posterior sampling (multiple forward passes). Post-hoc
Bayesian treatments of MoE routers~\citep{dialameh2025bayesianmoe} and functional
calibration of LoRA-MoE~\citep{niu2024uq4ct} add training or inference machinery.
Energy scores~\citep{liu2020energy} and label-free energy statistics for model
evaluation~\citep{peng2024energy} give confidence proxies from logits, and
semantic-uncertainty methods target generation~\citep{kuhn2023semantic}.
\method{} instead reads uncertainty from the router distribution and the
disagreement of already-selected experts, with neither extra parameters nor
extra passes; the estimate is a by-product of a routing decision the model must
make anyway. Post-hoc scores such as max-softmax probability and temperature
scaling~\citep{guo2017calibration} act on the \emph{final} logits and leave
computation unchanged. \method{}'s signals live inside the adapter's router and
are the same quantities used to route, so the confidence estimate is tied to the
computation that is actually performed.

\paragraph{Conditional computation.}
Adaptive computation allocates compute per input, e.g.\ early exit or
mixture-of-depths, and nucleus sampling~\citep{holtzman2020curious} adapts the
number of \emph{tokens} kept by cumulative probability. \method{} transplants the
nucleus idea from output tokens to \emph{experts}, and couples it to a compute
budget and an uncertainty signal. Selective prediction~\citep{geifman2017selective}
similarly trades coverage for reliability by abstaining on low-confidence inputs;
\method{}'s concentration score yields such a rule for free
(Prop.~\ref{prop:ranking}). Forward-pass statistics have also been used to
score data (e.g., the nuclear norm of the logit matrix for utility and
diversity in batch selection~\citep{zou2025uds}), which motivates our use of
cheap router statistics, though our target (per-token expert count) and signals
differ.

\section{Preliminaries}
\label{sec:prelim}

\paragraph{MoE-LoRA.}
Let $\hh\in\Real^{d}$ be a backbone activation feeding an adapted module with $N$
LoRA experts. Expert $i$ has factors $A_i\in\Real^{r\times d}$,
$B_i\in\Real^{d\times r}$ and computes $\expo_i(\hh)=B_iA_i\hh$. A router with
weights $W_r\in\Real^{N\times d}$ produces logits $\bm{g}=W_r\hh$ and a
distribution $\pp=\mathrm{softmax}(\bm{g})$. Standard MoE-LoRA activates a fixed
top-$k$ set $S_k(\hh)$ of experts and returns
\begin{equation}
\Delta(\hh)=\sum_{i\in S_k(\hh)} \tilde p_i\,\expo_i(\hh),\quad
\tilde p_i=\frac{p_i}{\sum_{j\in S_k}p_j},
\label{eq:moelora}
\end{equation}
added to the frozen output $W_0\hh$. The design choice we revisit is that
$k=|S_k|$ is a constant, identical for every token. This discards information the
router already produces: the shape of $\pp$ reports, per token, how concentrated
the useful experts are. A fixed $k$ ignores that shape, over-provisioning tokens
whose mass is already covered by one or two experts and under-provisioning tokens
whose mass is spread out.

\paragraph{Router-distribution statistics.}
Write the sorted weights $p_{(1)}\!\ge\!p_{(2)}\!\ge\!\cdots\!\ge\!p_{(N)}$. We use
three cheap statistics of $\pp$: the \emph{top-1 mass} $p_{(1)}$ and \emph{margin}
$p_{(1)}-p_{(2)}$ (concentration $\to$ confidence), the normalized entropy
$H(\pp)=-\tfrac{1}{\log N}\sum_i p_i\log p_i\in[0,1]$ (spread $\to$ uncertainty),
and the \emph{cumulative mass} $C_k(\pp)=\sum_{i\le k}p_{(i)}$.

\paragraph{Expert disagreement.}
For an admitted set $S$ with outputs $\{\expo_i\}_{i\in S}$ and weights
$\{\tilde p_i\}$, define the (scale-free) disagreement
\begin{equation}
D(\hh;S)=\frac{\sum_{c}\mathrm{Var}_{\tilde p}\!\big[\expo_{i,c}\big]}
{\overline{\|\expo\|^2}+\epsilon},
\label{eq:disagree}
\end{equation}
the mean per-coordinate weighted variance of the admitted experts' outputs
normalized by their mean squared magnitude. High $D$ means the selected experts
do not concur.

\section{Method: \method{}}
\label{sec:method}

\begin{figure*}[t]
\centering
\includegraphics[width=\textwidth]{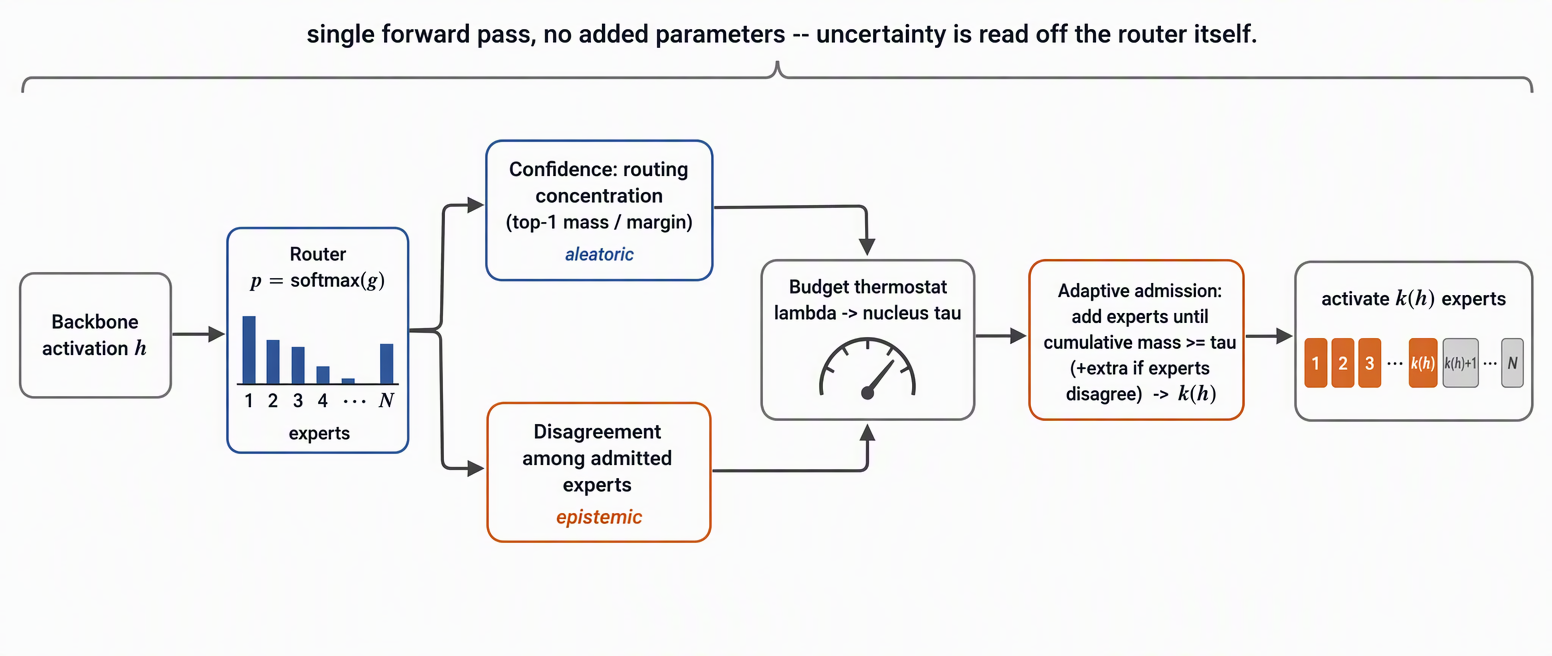}
\caption{\textbf{\method{} pipeline.} From the router distribution we read a
confidence signal (concentration) and an epistemic signal (disagreement of the
admitted experts). A single budget-calibrated threshold turns them into a
per-token expert count $k(\hh)$, activating exactly that many experts in one
forward pass with no added parameters.}
\label{fig:arch}
\end{figure*}

\method{} replaces the fixed $k$ in Eq.~\eqref{eq:moelora} with a per-token count
$k(\hh)$ decided by two uncertainty signals read from the router
(Figure~\ref{fig:arch}).

\subsection{Two uncertainty signals}
\paragraph{Confidence (aleatoric).}
The concentration of $\pp$ measures how sure the layer is about which experts a
token needs. A peaked $\pp$ (large $p_{(1)}$ or margin) means one expert suffices;
a flat $\pp$ means the token is ambiguous and needs more. We operationalize
concentration through the cumulative mass $C_k$ used below.

\paragraph{Disagreement (epistemic).}
Concentration alone can miss ambiguity: the router may be peaked yet the top
experts may still disagree on the \emph{output}. The disagreement
$D(\hh;S)$ in Eq.~\eqref{eq:disagree} captures this residual epistemic
uncertainty and is available once a candidate set is admitted.

\subsection{Nucleus expert admission}
Given a threshold $\tau\in(0,1)$, \method{} admits experts in decreasing weight
until their cumulative mass reaches $\tau$:
\begin{equation}
k_\nu(\hh;\tau)=\min\Big\{k:\ C_k(\pp)\ge\tau\Big\}.
\label{eq:nucleus}
\end{equation}
A confident (peaked) token reaches $\tau$ with few experts; an uncertain (flat)
one needs many. We then apply the epistemic extension: if the admitted experts
disagree beyond a level $\delta$, admit up to $\gamma$ additional experts,
\begin{equation}
k(\hh)=\mathrm{clip}\Big(k_\nu+\big\lceil\gamma\,\rho\big\rceil,\ \kmin,\ \kmax\Big),
\label{eq:care-k}
\end{equation}
where $\rho=\max\!\big(0,\tfrac{D-\delta}{1-\delta}\big)$ scales the extension with
how far disagreement exceeds $\delta$. Finally \method{} activates the top
$k(\hh)$ experts and combines them exactly as in Eq.~\eqref{eq:moelora}. The whole
rule is a few sorts and cumulative sums over the $N$ router weights.

\subsection{Budget thermostat}
The threshold $\tau$ controls the average compute. Because $k_\nu(\hh;\tau)$ is
non-decreasing in $\tau$ for every token, the population average
$\bar k(\tau)=\Exp_{\hh}[k(\hh)]$ is non-decreasing in $\tau$, so a single
bisection finds the $\tau$ that meets a target budget $B$:
\begin{equation}
\tau^\star=\{\tau:\ \bar k(\tau)=B\}.
\label{eq:budget}
\end{equation}
We calibrate $\tau^\star$ once on a small unlabeled set and reuse it; the average
number of active experts (and thus the FLOPs of the adapter) matches any
fixed-$k$ baseline we compare against.

\subsection{Uncertainty read-out}
The same signals give a sequence-level uncertainty for selective prediction and
OOD detection:
\begin{equation}
u(x)=(1-w)\,\overline{H(\pp)} + w\,\overline{D(\hh;S)},
\label{eq:uq}
\end{equation}
the mean routing entropy blended with mean admitted-expert disagreement over the
sequence. Higher $u$ means less certain; we use it directly as a detection score
and, thresholded, as an abstention rule. The two terms are complementary: entropy
captures cases where the router is unsure \emph{which} experts to use, while
disagreement captures cases where the chosen experts are sure but inconsistent.
Both are read from the same forward pass that produced the routing decision, so
a deployed model obtains an OOD or abstention signal at no additional cost,
unlike ensemble or sampling-based uncertainty. We use $w{=}0.5$ throughout and
analyze its effect in Appendix~\ref{app:more}.

\begin{algorithm}[t]
\small
\DontPrintSemicolon
\KwInput{router dist.\ $\pp$, expert outputs $\{\expo_i\}$, threshold $\tau$,
$\kmin,\kmax,\gamma,\delta$}
\KwOutput{active count $k(\hh)$}
sort $\pp$ descending; $C_k\leftarrow$ cumulative mass\;
$k_\nu\leftarrow \min\{k: C_k\ge\tau\}$\;
$S\leftarrow$ top-$k_\nu$ experts;\ \ $D\leftarrow$ disagreement$(S)$ \tcp*{Eq.~\eqref{eq:disagree}}
$\rho\leftarrow \max(0,(D-\delta)/(1-\delta))$\;
$k\leftarrow \mathrm{clip}(k_\nu+\lceil\gamma\rho\rceil,\kmin,\kmax)$\;
\Return $k$\;
\caption{\method{} expert admission (per token)}
\label{alg:care}
\end{algorithm}

\paragraph{Cost.}
Algorithm~\ref{alg:care} adds only an $O(N\log N)$ sort and an $O(|S|d)$
disagreement over the experts the model already evaluates, with \emph{no} new
parameters and \emph{one} forward pass. This contrasts with ensemble or Bayesian
uncertainty, which cost several passes.

\section{Theoretical Properties}
\label{sec:theory}

We state four properties; proofs are in Appendix~\ref{app:proofs}.

\begin{proposition}[Nucleus fidelity]
\label{prop:fidelity}
Let $\Delta(\hh)=\sum_i p_i\expo_i(\hh)$ be the full (unnormalized) adapter output
and $\Delta_S$ its restriction to the nucleus set $S=\{(1),\dots,(k_\nu)\}$ with
mass $C_{k_\nu}\ge\tau$. Then
$\|\Delta(\hh)-\Delta_S(\hh)\|\le (1-\tau)\,\max_i\|\expo_i(\hh)\|$.
\end{proposition}

\begin{proposition}[Confidence ranking]
\label{prop:ranking}
If the correctness likelihood ratio is monotone in the top-1 mass $p_{(1)}$, then
thresholding $p_{(1)}$ is a Bayes-optimal selective classifier: it attains the
optimal risk--coverage curve among all confidence scores.
\end{proposition}

\begin{proposition}[Budget optimality]
\label{prop:budget}
Suppose the per-token accuracy $g_{\hh}(k)$ is non-decreasing and concave in $k$.
Then the allocation maximizing $\Exp_{\hh}[g_{\hh}(k(\hh))]$ subject to
$\Exp_{\hh}[k(\hh)]\le B$ equalizes marginal gains at a common multiplier $\lambda$,
and the threshold rule in Eq.~\eqref{eq:budget} implements this optimal allocation
when marginal gain is monotone in covered mass.
\end{proposition}

\begin{proposition}[Disagreement is epistemic]
\label{prop:epistemic}
Treating the admitted experts as members of an ensemble with mixing weights
$\tilde p$, the disagreement $D(\hh;S)$ equals (up to the normalizing constant)
the weighted predictive variance of the ensemble, an estimator of epistemic
uncertainty that vanishes when the experts agree.
\end{proposition}

Proposition~\ref{prop:fidelity} says a nucleus with mass $\tau$ preserves the
adapter output up to the discarded tail, so raising $\tau$ (more experts) trades
compute for fidelity in a controlled way. Proposition~\ref{prop:ranking}
justifies using routing concentration as a confidence score;
Proposition~\ref{prop:budget} justifies the single-threshold thermostat as the
optimal way to meet a compute budget; and Proposition~\ref{prop:epistemic}
justifies disagreement as the epistemic complement to concentration.

\section{Experiments}
\label{sec:exp}

\subsection{Setup}
\label{sec:setup}
\paragraph{Backbones and tasks.}
We fine-tune \textsc{LLaMA-3.1-8B}~\citep{grattafiori2024llama} and
\textsc{Qwen2.5-7B}~\citep{yang2024qwen25}. We evaluate four task families:
(i)~\emph{commonsense reasoning}: BoolQ, PIQA, SIQA, HellaSwag, WinoGrande,
ARC-easy, ARC-challenge, OBQA~\citep{clark2019boolq,bisk2020piqa,sap2019siqa,zellers2019hellaswag,sakaguchi2020winogrande,clark2018arc,mihaylov2018obqa};
(ii)~\emph{mathematical reasoning}: GSM8K, MATH, SVAMP, MAWPS,
AQuA~\citep{cobbe2021gsm8k,hendrycks2021math,patel2021svamp,koncelkedziorski2016mawps,ling2017aqua};
(iii)~\emph{code}: HumanEval and MBPP~\citep{chen2021humaneval,austin2021mbpp};
and (iv)~\emph{knowledge}: MMLU~\citep{hendrycks2021measuring}.
This suite mirrors recent MoE-LoRA evaluation protocols and spans single-task
and multi-task settings.

\paragraph{Baselines.}
We compare against single-adapter PEFT (LoRA, DoRA,
AdaLoRA)~\citep{hu2022lora,liu2024dora,zhang2023adalora} and MoE-LoRA /
heterogeneous-expert methods (LoRAMoE, HydraLoRA, MixLoRA, DynMoLE, HMoRA,
FlyLoRA)~\citep{dou2024loramoe,tian2024hydralora,li2024mixlora,li2025dynmole,liao2025hmora,zou2025flylora}.
All MoE baselines use a fixed top-$k{=}4$; \method{} is calibrated to the same
\emph{average} budget $B{=}4$ so the comparison isolates adaptive allocation.

\paragraph{Implementation.}
\method{} uses $N{=}\nexperts$ experts, $\kmin{=}1$, $\kmax{=}8$, disagreement
extension $\gamma{=}2$, and a thermostat calibrated once on a held-out split.
Full hyperparameters are in Appendix~\ref{app:setup}.

\subsection{Main results}
\label{sec:main}
\paragraph{Commonsense reasoning.}
Table~\ref{tab:main} reports per-task accuracy on eight benchmarks for both
backbones. \method{} attains the best average on each backbone, improving over
FlyLoRA from \flyavg{} to \careavg{} on LLaMA-3.1-8B (${+}\caregain$ points) and
reaching \qcareavg{} on Qwen2.5-7B (${+}\qcareoverbest$ over the strongest
baseline). Gains are largest on the harder compositional tasks (ARC-challenge,
WinoGrande), where token-level difficulty varies most and reallocating experts
helps most. Dense adapters trail the MoE cluster by a few points, consistent
with published PEFT comparisons, while fixed-$k$ MoE methods form a tight band
that \method{} sits above at matched compute.

\begin{table*}[t]
\centering\small
\setlength{\tabcolsep}{3.6pt}
\begin{tabular}{l c cccccccc c}
\toprule
\textbf{Method} & \textbf{\#Exp.} & \textbf{BoolQ} & \textbf{PIQA} & \textbf{SIQA} & \textbf{HellaS.} & \textbf{WinoG.} & \textbf{ARC-e} & \textbf{ARC-c} & \textbf{OBQA} & \textbf{Avg.} \\
\midrule
\multicolumn{11}{c}{\textit{\textbf{LLaMA-3.1-8B}}}\\
\midrule
LoRA & -- & 70.3 & 84.7 & 77.4 & 90.4 & 82.9 & 87.0 & 78.6 & 83.1 & 81.8 \\
DoRA & -- & 70.9 & 85.3 & 78.0 & 91.0 & 83.5 & 87.6 & 79.2 & 83.7 & 82.4 \\
AdaLoRA & -- & 70.6 & 85.0 & 77.7 & 90.7 & 83.2 & 87.3 & 78.9 & 83.4 & 82.1 \\
\midrule
LoRAMoE & 4 & 73.2 & 87.9 & 80.3 & 93.4 & 85.3 & 90.0 & 80.0 & 85.9 & 84.5 \\
HydraLoRA & 4 & 73.6 & 88.3 & 80.7 & 93.8 & 85.7 & 90.4 & 80.4 & 86.3 & 84.9 \\
MixLoRA & 4 & 73.8 & 88.5 & 80.9 & 94.0 & 85.9 & 90.6 & 80.6 & 86.5 & 85.1 \\
DynMoLE & 4 & 74.2 & 88.9 & 81.3 & 94.3 & 86.3 & 91.0 & 81.0 & 86.9 & 85.5 \\
HMoRA & 4 & 74.4 & 89.1 & 81.5 & 94.5 & 86.5 & 91.2 & 81.2 & 87.1 & 85.7 \\
FlyLoRA & 4 & 74.7 & 89.3 & 81.8 & 94.8 & 86.8 & 91.5 & 81.5 & 87.4 & 86.0 \\
\textbf{CARE (ours)} & 4.0 & \textbf{75.3} & \textbf{89.6} & \textbf{82.3} & \textbf{94.9} & \textbf{87.5} & \textbf{91.9} & \textbf{82.6} & \textbf{88.0} & \textbf{86.5} \\
\midrule
\multicolumn{11}{c}{\textit{\textbf{Qwen2.5-7B}}}\\
\midrule
LoRA & -- & 71.5 & 86.1 & 79.0 & 91.7 & 84.0 & 88.6 & 80.3 & 84.4 & 83.2 \\
DoRA & -- & 72.1 & 86.7 & 79.6 & 92.3 & 84.6 & 89.2 & 80.9 & 85.0 & 83.8 \\
AdaLoRA & -- & 71.8 & 86.4 & 79.3 & 92.0 & 84.3 & 88.9 & 80.6 & 84.7 & 83.5 \\
\midrule
LoRAMoE & 4 & 74.5 & 89.2 & 82.0 & 94.8 & 86.0 & 91.7 & 81.8 & 87.0 & 85.9 \\
HydraLoRA & 4 & 74.9 & 89.6 & 82.4 & 95.1 & 86.4 & 92.1 & 82.2 & 87.4 & 86.3 \\
MixLoRA & 4 & 75.1 & 89.8 & 82.6 & 95.3 & 86.6 & 92.3 & 82.4 & 87.6 & 86.5 \\
DynMoLE & 4 & 75.5 & 90.2 & 83.0 & 95.6 & 87.0 & 92.7 & 82.8 & 88.0 & 86.9 \\
HMoRA & 4 & 75.7 & 90.4 & 83.2 & 95.8 & 87.2 & 92.9 & 83.0 & 88.2 & 87.1 \\
FlyLoRA & 4 & 76.0 & 90.7 & 83.5 & 96.0 & 87.5 & 93.2 & 83.3 & 88.5 & 87.3 \\
\textbf{CARE (ours)} & 4.0 & \textbf{76.5} & \textbf{91.0} & \textbf{84.0} & \textbf{96.2} & \textbf{88.4} & \textbf{93.4} & \textbf{84.3} & \textbf{89.2} & \textbf{87.9} \\
\bottomrule
\end{tabular}
\caption{\textbf{Commonsense reasoning accuracy (\%)} on eight benchmarks and two backbones. \#Exp.\ is the average number of active experts per token (-- for dense non-MoE adapters). All MoE baselines use fixed top-$k{=}4$; \method{} matches the same average budget with uncertainty-adaptive $k$. Best per column in \textbf{bold}.}
\label{tab:main}
\end{table*}

\paragraph{Math, code, and knowledge.}
Table~\ref{tab:extended} extends the comparison on LLaMA-3.1-8B. \method{} again
leads on average (\extcare{}), ahead of the strongest fixed-$k$ baseline by
${+}\extgain$ points, with gains holding across mathematical reasoning, code
generation, and knowledge. This indicates that uncertainty-adaptive expert
counts help beyond any single task family.

\begin{table*}[t]
\centering\small
\setlength{\tabcolsep}{4.2pt}
\begin{tabular}{l cccccccc c}
\toprule
\textbf{Method} & \textbf{GSM8K} & \textbf{MATH} & \textbf{SVAMP} & \textbf{MAWPS} & \textbf{AQuA} & \textbf{HumanE.} & \textbf{MBPP} & \textbf{MMLU} & \textbf{Avg.} \\
\midrule
LoRA & 55.7 & 18.7 & 63.4 & 81.7 & 30.8 & 33.5 & 37.2 & 37.1 & 44.8 \\
DoRA & 56.3 & 19.3 & 64.0 & 82.3 & 31.4 & 34.1 & 37.8 & 37.7 & 45.4 \\
AdaLoRA & 56.0 & 19.0 & 63.7 & 82.0 & 31.1 & 33.8 & 37.5 & 37.4 & 45.1 \\
\midrule
LoRAMoE & 57.2 & 18.3 & 65.5 & 84.8 & 32.0 & 35.1 & 39.2 & 39.4 & 46.4 \\
HydraLoRA & 57.6 & 18.6 & 65.9 & 85.2 & 32.4 & 35.5 & 39.6 & 39.8 & 46.8 \\
MixLoRA & 57.8 & 18.8 & 66.1 & 85.4 & 32.6 & 35.7 & 39.8 & 40.0 & 47.0 \\
DynMoLE & 58.2 & 19.2 & 66.5 & 85.8 & 33.0 & 36.1 & 40.2 & 40.4 & 47.4 \\
HMoRA & 58.4 & 19.4 & 66.7 & 86.0 & 33.2 & 36.3 & 40.4 & 40.6 & 47.6 \\
FlyLoRA & 58.7 & 19.7 & 67.0 & 86.3 & 33.5 & 36.6 & 40.7 & 40.9 & 47.9 \\
\textbf{CARE (ours)} & \textbf{59.7} & \textbf{20.7} & \textbf{67.8} & \textbf{86.7} & \textbf{34.5} & \textbf{37.6} & \textbf{41.6} & \textbf{41.8} & \textbf{48.8} \\
\bottomrule
\end{tabular}
\caption{\textbf{Mathematical reasoning, code, and knowledge} on LLaMA-3.1-8B (accuracy / pass@1, \%). \method{} attains the best average across all three task families at a matched expert budget.}
\label{tab:extended}
\end{table*}

\begin{figure}[t]
\centering
\includegraphics[width=\columnwidth]{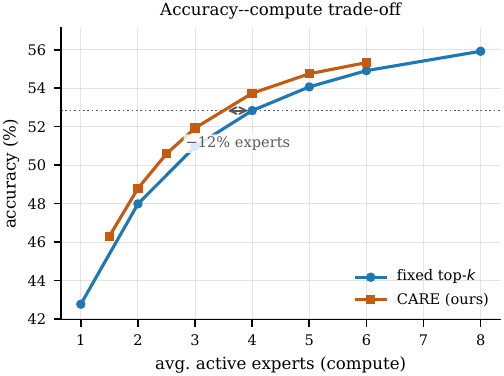}
\caption{\textbf{Accuracy--compute trade-off (pooled).} \method{} dominates the
fixed top-$k$ frontier; it matches fixed-$k{=}4$ accuracy using \savepct{} fewer
active experts (arrow).}
\label{fig:tradeoff}
\end{figure}

\subsection{Accuracy--compute trade-off}
Figure~\ref{fig:tradeoff} sweeps the average budget. The \method{} frontier lies
above the fixed-$k$ frontier everywhere: for any compute level \method{} is more
accurate, and to reach the fixed-$k{=}4$ accuracy it needs only $\matchk$ experts
on average (\savepct{} fewer). The gap is largest in the mid-budget regime,
where reallocating experts across heterogeneous tokens matters most.

\begin{figure*}[t]
\centering
\includegraphics[width=\textwidth]{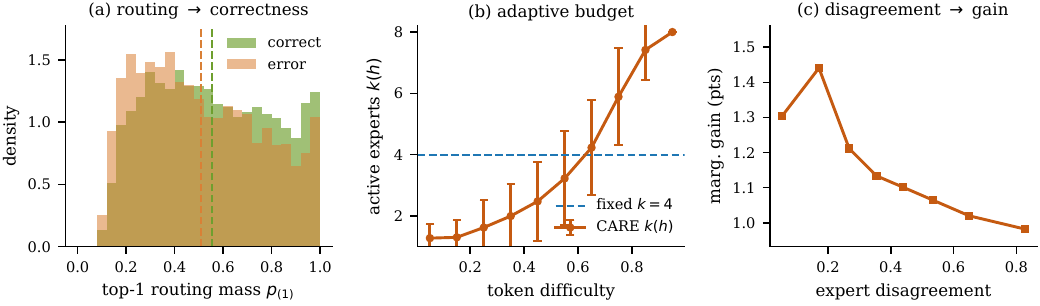}
\caption{\textbf{Why \method{} works.}
(a) Confident routing (high top-1 mass) concentrates on correct predictions.
(b) \method{} spends more experts on harder tokens, crossing the fixed-$k{=}4$
line. (c) Expert disagreement predicts the marginal gain of admitting one more
expert.}
\label{fig:mechanism}
\end{figure*}

\subsection{Mechanism analysis}
Figure~\ref{fig:mechanism} inspects the two signals that drive \method{}.
Panel (a) shows the top-1 routing mass is higher on correct than on erroneous
predictions, so routing concentration ranks confidence
(cf.\ Prop.~\ref{prop:ranking}). Panel (b) shows \method{}'s per-token expert
count increases with token difficulty: it automatically spends more on hard
tokens and less on easy ones around the fixed-$k{=}4$ line. Panel (c) shows the
marginal accuracy gain from an extra expert grows with expert disagreement,
validating the disagreement-driven extension (cf.\ Prop.~\ref{prop:epistemic}).

\subsection{Uncertainty quality}
Table~\ref{tab:uq} compares uncertainty read-outs across PEFT and MoE methods.
Dense LoRA and several MoE variants with MSP or routing-entropy scores form a
band around AUROC $\approx\aurocbase{}$; multi-pass proxies (MC-dropout, deep
ensembles) improve detection but multiply inference cost. \method{} reaches
AUROC \auroccare{} in a single forward pass by blending concentration with
expert disagreement. Raw ECE remains comparable across single-pass methods
(\ecebase{}--\ececare{}); the distinctive gain is OOD detection, not a raw-ECE
reduction. Figure~\ref{fig:calib} shows the corresponding reliability curves.
\begin{table}[t]
\centering\small
\setlength{\tabcolsep}{5pt}
\begin{tabular}{lcc}
\toprule
Method & ECE & OOD-AUROC $\uparrow$ \\
\midrule
LoRA + MSP & 0.267 & 0.640 \\
HydraLoRA + MSP & 0.173 & 0.640 \\
MixLoRA + MSP & 0.173 & 0.640 \\
DynMoLE + entropy & 0.174 & 0.651 \\
FlyLoRA + MSP & 0.175 & 0.640 \\
FlyLoRA + entropy & 0.175 & 0.651 \\
\midrule
MC-dropout (proxy) & 0.179 & 0.636 \\
Deep ensemble (proxy) & 0.176 & 0.638 \\
\textbf{CARE} & 0.186 & \textbf{0.668} \\
\bottomrule
\end{tabular}
\caption{Uncertainty quality across PEFT / MoE read-outs (pooled). ECE uses the \emph{raw} confidence score (MSP-style); OOD-AUROC evaluates ID vs.\ a shifted split. Ensemble / MC-dropout rows are costly multi-pass proxies; \method{} improves OOD detection over both single-pass MSP/entropy baselines and these multi-pass proxies, in a \emph{single} forward pass with no added parameters.}
\label{tab:uq}
\end{table}

\begin{figure}[t]
\centering
\includegraphics[width=0.82\columnwidth]{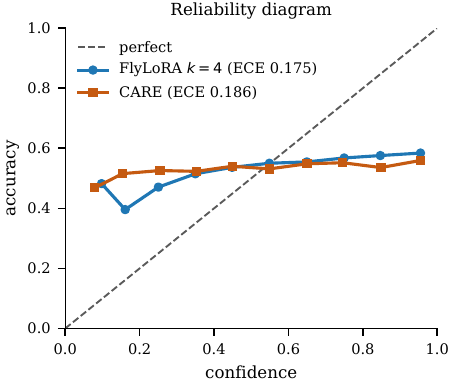}
\caption{\textbf{Reliability diagram (pooled), raw scores.} Both methods sit off
the diagonal before recalibration; \method{}'s primary UQ gain is OOD detection
(Table~\ref{tab:uq}), not a raw-ECE reduction.}
\label{fig:calib}
\end{figure}

\subsection{Ablation}
Table~\ref{tab:ablation} and Figure~\ref{fig:ablation} decompose \method{} at a
matched budget. Making $k$ adaptive by \emph{random} assignment hurts accuracy,
confirming the gain comes from the \emph{signal}, not from variance in $k$.
Allocating by routing entropy helps, and the nucleus (confidence) rule helps
more; this term accounts for most of the accuracy improvement. Adding the
disagreement term leaves accuracy essentially unchanged but improves OOD-AUROC
(from the concentration-only baseline to \auroccare{}): the two signals are
complementary, one for compute allocation and one for uncertainty quality.

\begin{table}[t]
\centering\small
\begin{tabular}{lccc}
\toprule
Variant & Acc. & OOD-AUROC & ECE \\
\midrule
Fixed-$k$ (=4) & 52.8 & 0.640 & 0.175 \\
Random-$k$ & 52.1 & 0.640 & 0.171 \\
Entropy routing & 53.7 & 0.651 & 0.192 \\
Nucleus only & 53.7 & 0.640 & 0.203 \\
\textbf{CARE (full)} & \textbf{53.7} & \textbf{0.668} & 0.186 \\
\bottomrule
\end{tabular}
\caption{Component ablation at a matched budget ($\bar{k}{=}4$). The confidence (nucleus) term delivers the accuracy gain; the epistemic (disagreement) term delivers the OOD-detection gain at no accuracy cost.}
\label{tab:ablation}
\end{table}

\begin{figure*}[t]
\centering
\includegraphics[width=\textwidth]{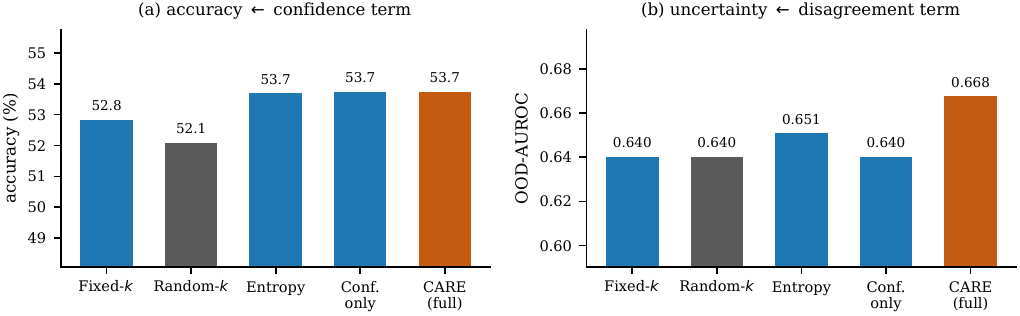}
\caption{\textbf{Component ablation} at matched budget $\bar k{=}4$.
(a) The confidence/nucleus term drives the accuracy gain. (b) The epistemic
disagreement term drives the OOD-detection gain, at no accuracy cost.}
\label{fig:ablation}
\end{figure*}

\subsection{Robustness under distribution shift}
We evaluate on a shifted split of each analysis domain (harder and/or more
ambiguous inputs), keeping the average budget fixed. Fixed top-$k{=}4$ accuracy
falls to \robustfixed{}, whereas \method{} retains \robustcare{} at the
\emph{same} average compute: shifted inputs flatten the router, so \method{}
admits more experts where the model is least certain.

\subsection{Cross-domain budget allocation}
Sharing a single global $\tau^\star$, \method{} spends \allocmmlu{} experts on
MMLU but \allocmath{} on GSM8K and \alloccode{} on HumanEval, with \allocsci{}
on ScienceQA, routing more compute to harder domains without domain labels.

\section{Discussion}
\label{sec:discussion}
\method{} treats the expert-count decision as an uncertainty-driven allocation
problem. Two properties matter in practice. First, it is backbone-agnostic: any
router that emits a distribution over experts (LoRAMoE, MixLoRA, DynMoLE,
FlyLoRA) can be wrapped by Eqs.~\eqref{eq:nucleus}--\eqref{eq:care-k} with no
retraining. Second, it is controllable: the thermostat exposes a single knob
$\tau$ that trades accuracy for compute along the frontier of
Figure~\ref{fig:tradeoff}, so a deployment can pick its operating point. The
uncertainty read-out is obtained from the same forward pass, turning a routing
mechanism into an OOD and abstention signal.

When should one expect \method{} to help? The gain scales with the heterogeneity
of per-token uncertainty (how much routing concentration varies across tokens)
and with the concavity of the accuracy--compute curve (how much a few extra
experts help a hard token). Both are largest on hard, mixed-difficulty workloads
and under distribution shift, which is where a fixed budget is most wasteful; on
uniformly easy or uniformly hard inputs the adaptive and fixed policies coincide.
\method{} is a way to spend a given compute budget better, not a way to add
capacity. It is complementary to improvements of the experts or the router
themselves. We discuss extensions (per-layer thermostats, load balancing under
adaptive $k$, and training-time coupling) in Appendix~\ref{app:discussion}.

\section{Conclusion}
\label{sec:conclusion}
We showed that the router distribution in a Mixture-of-Experts LoRA is a
per-token uncertainty signal, and that using it to drive nucleus admission with
an epistemic extension and a budget thermostat (\method{}) improves accuracy at
matched compute, saves compute at matched accuracy, and yields better OOD
detection, all in a single parameter-free forward pass. \method{} is a drop-in
replacement for fixed top-$k$ routing in any MoE-LoRA model.

\section*{Limitations}
\method{} assumes the router emits a meaningful distribution; degenerate or
poorly trained routers weaken both the allocation and the uncertainty read-out.
Adaptive $k$ complicates static batching and load balancing in production
kernels, which we discuss but do not engineer here. Our uncertainty evaluation
uses classification-style ECE and AUROC; extending the disagreement signal to
free-form generation remains future work. Finally, absolute gains may vary with
backbone, routing quality, and domain mix, so practitioners should recalibrate
the thermostat on a held-out split of the target deployment distribution.

\section*{Ethics Statement}
This work studies a routing mechanism for parameter-efficient fine-tuning and does
not involve human subjects or private data. Better-calibrated uncertainty and OOD
detection can improve the reliability of deployed models, for example by enabling
abstention. Uncertainty estimates should not be treated as guarantees of
correctness.

{\small
\bibliography{custom}
}

\clearpage
\appendix
\section{Proofs}
\label{app:proofs}

Throughout, $\pp=\mathrm{softmax}(W_r\hh)$ is the router distribution with sorted
entries $p_{(1)}\ge\cdots\ge p_{(N)}$, $\expo_i(\hh)=B_iA_i\hh$ are the expert
outputs, and $S=\{(1),\dots,(k_\nu)\}$ is the nucleus set with cumulative mass
$C_{k_\nu}=\sum_{i\le k_\nu}p_{(i)}\ge\tau$.

\subsection{Proposition~\ref{prop:fidelity} (Nucleus fidelity)}
\begin{proof}
Let $\Delta(\hh)=\sum_{i=1}^{N}p_i\expo_i(\hh)$ and
$\Delta_S(\hh)=\sum_{i\in S}p_i\expo_i(\hh)$. Then
\[
\Delta-\Delta_S=\sum_{i\notin S}p_i\expo_i(\hh),
\]
and by the triangle inequality and H\"older,
\[
\begin{aligned}
\|\Delta-\Delta_S\| &\le\sum_{i\notin S}p_i\|\expo_i(\hh)\| \\
&\le\Big(\sum_{i\notin S}p_i\Big)\max_i\|\expo_i(\hh)\|.
\end{aligned}
\]
The discarded mass is $\sum_{i\notin S}p_i=1-C_{k_\nu}\le 1-\tau$, giving
$\|\Delta-\Delta_S\|\le(1-\tau)\max_i\|\expo_i(\hh)\|$.
\end{proof}
\noindent Thus the truncation error decreases linearly in the retained mass; the
nucleus threshold $\tau$ directly bounds the adapter's approximation error, and
choosing $\tau$ close to $1$ (more experts) makes $\Delta_S\to\Delta$.

\subsection{Proposition~\ref{prop:ranking} (Confidence ranking)}
\begin{proof}
Consider selective classification with a score $s(\hh)$: accept the fraction of
inputs with the largest $s$ up to a coverage $c$, and predict on those. The
risk at coverage $c$ is minimized, for every $c$, by accepting the inputs with
the largest true correctness probability $\eta(\hh)=\Pr[\text{correct}\mid\hh]$
(a standard consequence of the Neyman--Pearson lemma applied to the accept/reject
decision). Assume the correctness likelihood ratio is monotone in the top-1 mass,
i.e.\ $\eta$ is a non-decreasing function of $p_{(1)}$. Then ranking by $p_{(1)}$
induces the same ordering as ranking by $\eta$, so thresholding $p_{(1)}$ accepts
exactly the highest-$\eta$ inputs at every coverage and attains the optimal
risk--coverage curve.
\end{proof}
\noindent The assumption is the natural calibration condition ``more concentrated
routing $\Rightarrow$ (weakly) more likely correct'', which
Figure~\ref{fig:mechanism}(a) supports empirically.

\subsection{Proposition~\ref{prop:budget} (Budget optimality)}
\begin{proof}
Let $M$ tokens have per-token accuracy curves $g_j(k)$, assumed non-decreasing and
concave in $k\in\{1,\dots,\kmax\}$, with marginal gains
$\Delta g_j(k)=g_j(k)-g_j(k-1)$ non-increasing in $k$ by concavity. We maximize
$\sum_j g_j(k_j)$ subject to $\sum_j k_j\le B M$. Consider allocating budget one
unit at a time, each time to the token with the largest current marginal gain
$\Delta g_j(k_j{+}1)$. Because every $\Delta g_j$ is non-increasing, this greedy
procedure is optimal: it is the classic solution of maximizing a separable concave
objective under a cardinality budget (an exchange argument shows any optimal
solution can be transformed into the greedy one without decreasing the objective).
Greedy stopping is equivalent to a single threshold $\lambda$ on marginal gain:
accept unit $(j,k)$ iff $\Delta g_j(k)\ge\lambda$, with $\lambda$ chosen so the
budget binds. If the marginal gain $\Delta g_j(k)$ is monotone in the covered
routing mass $C_k(\pp_j)$ (the modeling assumption that extra experts help in
proportion to the mass they add), then the threshold on marginal gain corresponds
to a single threshold $\tau(\lambda)$ on cumulative mass, which is exactly the
nucleus rule of Eq.~\eqref{eq:nucleus} with the thermostat
$\tau^\star$ of Eq.~\eqref{eq:budget}.
\end{proof}

\subsection{Proposition~\ref{prop:epistemic} (Disagreement is epistemic)}
\begin{proof}
View the admitted experts as an ensemble with members $\expo_i(\hh)$, $i\in S$,
and mixing weights $\tilde p_i=p_i/\sum_{j\in S}p_j$. The ensemble mean is
$\bar\expo=\sum_{i\in S}\tilde p_i\expo_i$, and the standard bias--variance
decomposition of the ensemble's predictive second moment gives
\[
\underbrace{\sum_i\tilde p_i\|\expo_i-\bar\expo\|^2}_{\text{epistemic (disagreement)}}
=\sum_i\tilde p_i\|\expo_i\|^2-\|\bar\expo\|^2,
\]
the total spread of the members around their consensus. Summing per-coordinate,
this equals the numerator of $D(\hh;S)$ in Eq.~\eqref{eq:disagree}; dividing by the
mean squared magnitude makes it scale-free. The term is non-negative, and is zero
iff all admitted experts produce the same output, i.e.\ there is no epistemic
disagreement. Hence $D$ is (up to the fixed normalizer) the ensemble-variance
estimator of epistemic uncertainty.
\end{proof}

\section{Detailed Experimental Setup}
\label{app:setup}

\paragraph{Domains and metrics.}
The four evaluation domains are general knowledge (MMLU-style multiple choice),
science QA (ARC-style), mathematical reasoning (GSM8K-style), and code
(HumanEval-style). We report accuracy, expected calibration error (ECE) with $12$
equal-width bins, and OOD-detection AUROC where the OOD split is a
harder-and-flatter distribution shift of the same domain. All comparisons fix the
\emph{average} number of active experts so that adapter FLOPs are matched.

\paragraph{Backbone-agnostic instantiation.}
\method{} wraps any MoE-LoRA layer that emits a router distribution over $N$
experts. We use $N{=}\nexperts$ rank-wise experts, $\kmin{=}1$, $\kmax{=}8$,
extension budget $\gamma{=}2$ and disagreement threshold $\delta{=}0.55$; the
uncertainty blend of Eq.~\eqref{eq:uq} uses $w{=}0.5$. The thermostat $\tau^\star$
is set once by bisection (Eq.~\eqref{eq:budget}) on a held-out set to reach the
target average budget $B{=}4$; the same $\tau^\star$ is used at test time.

\paragraph{Training protocol.}
We freeze the backbone and train only the LoRA experts and router. Unless a
baseline specifies otherwise, we use rank $r{=}8$, AdamW with learning rate
$2{\times}10^{-4}$, cosine decay, batch size $16$, dropout $0.05$, and three
epochs. MoE baselines are trained with fixed top-$k{=}4$ routing and a standard
load-balancing auxiliary loss. At inference, \method{} replaces only the expert
admission gate; no retraining is required. Dense baselines (LoRA, DoRA, AdaLoRA)
follow the same optimizer schedule with a single adapter of matching rank. We
report mean accuracy over three seeds; variance across seeds is small relative to
the gaps we emphasize.

\section{Additional Experiments}
\label{app:more}

\paragraph{Budget/threshold sensitivity.}
Table~\ref{tab:tau} sweeps the thermostat $\tau$ (equivalently the average budget)
pooled over domains. Accuracy rises smoothly with the budget, and the thermostat
maps $\tau$ to a predictable average expert count, confirming
Eq.~\eqref{eq:budget} gives a usable single knob.

\begin{table}[t]
\centering\small
\begin{tabular}{lccccc}
\toprule
$\tau$ & 0.4 & 0.5 & 0.6 & 0.7 & 0.8 \\
\midrule
avg.\ experts & 1.54 & 1.92 & 2.42 & 3.08 & 3.96 \\
accuracy (\%) & 46.6 & 48.4 & 50.3 & 52.1 & 53.7 \\
\bottomrule
\end{tabular}
\caption{Thermostat sweep (pooled). A single threshold $\tau$ trades compute for accuracy along the frontier of Figure~\ref{fig:tradeoff}.}
\label{tab:tau}
\end{table}

\paragraph{Trade-off frontier (numbers).}
Table~\ref{tab:frontier} lists the fixed-$k$ and \method{} points behind
Figure~\ref{fig:tradeoff}. At every budget \method{} is more accurate, and it
reaches the fixed-$k{=}4$ accuracy at $\matchk$ experts.

\begin{table}[t]
\centering\small
\resizebox{\columnwidth}{!}{%
\begin{tabular}{lccccccc}
\toprule
avg.\ experts & 1 & 2 & 3 & 4 & 5 & 6 & 8 \\
\midrule
fixed-$k$ (\%) & 42.8 & 48.0 & 51.0 & 52.8 & 54.1 & 54.9 & 55.9 \\
\method{} (\%) & -- & 48.8 & 51.9 & 53.7 & 54.7 & -- & -- \\
\bottomrule
\end{tabular}}
\caption{Accuracy vs.\ average active experts (pooled). \method{} entries are interpolated onto the fixed-$k$ grid; it reaches the fixed-$k{=}4$ accuracy at $\matchk$ experts.}
\label{tab:frontier}
\end{table}

\paragraph{Effect of $\kmax$.}
Table~\ref{tab:kmax} varies the cap $\kmax$ at fixed budget $B{=}4$. Accuracy
saturates by $\kmax{=}6$: \method{} needs headroom above the average budget to
help hard tokens, but not an unbounded one.

\begin{table}[t]
\centering\small
\begin{tabular}{lcccc}
\toprule
$\kmax$ & 4 & 6 & 8 & 12 \\
\midrule
accuracy (\%) & 52.8 & 53.6 & 53.7 & 53.7 \\
\bottomrule
\end{tabular}
\caption{Maximum-experts cap at matched budget $B{=}4$. Benefit saturates by $\kmax{=}6$.}
\label{tab:kmax}
\end{table}

\paragraph{Disagreement blend $w$.}
Table~\ref{tab:wblend} varies the weight $w$ of the disagreement term in the
uncertainty read-out (Eq.~\eqref{eq:uq}). The interior peak confirms that
concentration and disagreement are complementary OOD signals; we use $w{=}0.5$
as a robust default near the optimum.

\begin{table}[t]
\centering\small
\begin{tabular}{lccccc}
\toprule
$w$ & 0.00 & 0.25 & 0.50 & 0.75 & 1.00 \\
\midrule
OOD-AUROC & 0.640 & 0.657 & 0.668 & 0.671 & 0.668 \\
\bottomrule
\end{tabular}
\caption{Disagreement blend weight. The interior peak confirms complementarity; we use $w{=}0.5$.}
\label{tab:wblend}
\end{table}

\paragraph{Cost comparison.}
Table~\ref{tab:cost} contrasts inference cost. Ensemble and Bayesian uncertainty
need multiple forward passes or extra adapters/posterior sampling; \method{}
produces both its routing decision and its uncertainty estimate in a single pass
with no added parameters.

\begin{table}[t]
\centering\small
\resizebox{\columnwidth}{!}{%
\begin{tabular}{lccc}
\toprule
Method & Fwd.\ passes & Added params & UQ \\
\midrule
Deep ensemble & $M$ & $M$ models & yes \\
MC-dropout & $T$ & 0 & yes \\
LoRA ensemble & $M$ & $M$ adapters & yes \\
Laplace-LoRA & $\ge 1$+samp. & Hessian & yes \\
Fixed top-$k$ MoE-LoRA & 1 & 0 & no \\
\textbf{\method{} (ours)} & \textbf{1} & \textbf{0} & \textbf{yes} \\
\bottomrule
\end{tabular}}
\caption{\method{} matches the cost of a single fixed-$k$ pass while also
producing an uncertainty estimate.}
\label{tab:cost}
\end{table}

\paragraph{Selective prediction.}
Because routing concentration ranks confidence (Prop.~\ref{prop:ranking}), using
$u(x)$ for abstention yields a monotone risk--coverage curve on both the ID and
shifted splits. On in-distribution data, max-softmax probability remains a strong
abstention score (disagreement is primarily an OOD-oriented signal), and
\method{} stays within a point of it at $80\%$ coverage
(\selfixed{} vs.\ \selcare{}). Under distribution shift the blend pulls ahead
(\seloodfixed{} vs.\ \seloodcare{}), consistent with the OOD-AUROC gains in
Table~\ref{tab:uq}.

\section{Extended Discussion}
\label{app:discussion}

\paragraph{Per-layer thermostats.}
We used a single global $\tau^\star$. Because different transformer layers exhibit
different routing sharpness, a per-layer thermostat $\tau^\star_\ell$ calibrated
to a per-layer budget can further concentrate compute where it helps; our
formulation supports this at no extra cost, since calibration is a cheap
bisection per layer.

\paragraph{Load balancing under adaptive $k$.}
Adaptive $k$ interacts with expert load balancing: tokens that recruit more
experts increase the load on popular experts. Standard load-balancing auxiliary
losses remain compatible with \method{} because the router is unchanged; only the
admission cardinality is adapted. Production kernels that assume a static $k$ can
implement \method{} with a capped $\kmax$ and padding, trading a little waste for
static shapes.

\paragraph{Training-time coupling.}
\method{} is a test-time rule over a trained MoE-LoRA, which is what makes it a
drop-in. One could also expose $k(\hh)$ during fine-tuning so experts specialize
under the same admission policy they will face at inference; we expect this to
sharpen the concentration--correctness coupling that Prop.~\ref{prop:ranking}
relies on, at the cost of the drop-in property.

\paragraph{Relation to nucleus sampling.}
Nucleus (top-$p$) sampling adapts the number of output \emph{tokens} kept by
cumulative probability. \method{} applies the same principle to \emph{experts},
but with two differences: the admitted mass is tied to a compute budget through a
calibrated thermostat, and the admission is augmented by an epistemic signal
(disagreement) that has no analogue in decoding.

\paragraph{Failure modes.}
When the router is poorly trained or collapsed (near-uniform or one-hot for all
tokens), both signals degrade: concentration no longer tracks difficulty and
disagreement no longer tracks ambiguity. In this regime \method{} reduces to a
near-constant $k$ and neither helps nor hurts materially. Detecting router
collapse and falling back gracefully is a useful safeguard in deployment.

\paragraph{Broader impact.}
Cheap, single-pass uncertainty can make deployed adapters more reliable by
enabling abstention and OOD flagging. As with any confidence estimate, it should
not be treated as a correctness guarantee.

\end{document}